\documentclass{article}

\usepackage{microtype}
\usepackage{graphicx}
\usepackage{subfigure}
\usepackage{booktabs} 
\usepackage{xcolor}
\usepackage{isomath}
\usepackage{amsmath}
\usepackage{bm}
\usepackage{amsthm}
\newtheorem{mythm}{Theorem}

\usepackage{hyperref}

\usepackage[accepted]{icml2019}
\usepackage{authblk}

\title{Lie Group Auto-Encoder}
\author[1]{Liyu Gong\thanks{liyu.gong@uky.edu}}
\author[1,2]{Qiang Cheng\thanks{Qiang.Cheng@uky.edu}}
\affil[1]{Institute for Biomedical Informatics, University of Kentucky, Lexington, KY, USA.}
\affil[2]{Department of Computer Science, University of Kentucky, Lexington, KY, USA.}
\date{}

\begin{document}

\maketitle

\begin{abstract}
  In this paper, we propose an auto-encoder based generative neural
  network model whose encoder compresses the inputs into vectors in
  the tangent space of a special Lie group manifold: upper triangular
  positive definite affine transform matrices (UTDATs). UTDATs are
  representations of Gaussian distributions and can straightforwardly
  generate Gaussian distributed samples. Therefore, the encoder is
  trained together with a decoder (generator) which takes Gaussian
  distributed latent vectors as input. Compared with related
  generative models such as variational auto-encoder, the proposed
  model incorporates the information on geometric properties of
  Gaussian distributions. As a special case, we derive an exponential
  mapping layer for diagonal Gaussian UTDATs which eliminates matrix
  exponential operator compared with general exponential mapping in
  Lie group theory. Moreover, we derive an intrinsic loss for UTDAT
  Lie group which can be calculated as l-2 loss in the tangent
  space. Furthermore, inspired by the Lie group theory, we propose to
  use the Lie algebra vectors rather than the raw parameters
  (e.g. mean) of Gaussian distributions as compressed representations
  of original inputs. Experimental results verity the effectiveness of
  the proposed new generative model and the benefits gained from the
  Lie group structural information of UTDATs.
\end{abstract}

\begin{figure}[ht]
  \centering
  \includegraphics[scale=0.47]{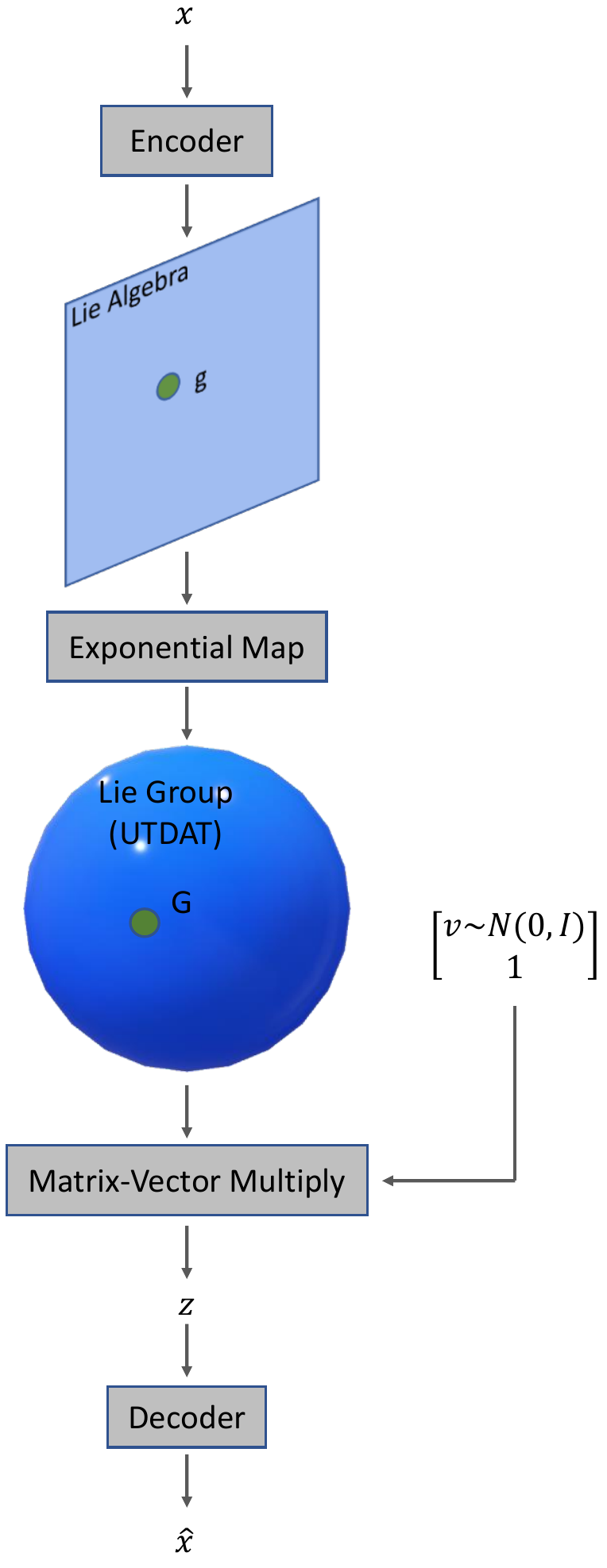}
  \caption{Overview of the proposed LGAE model. An input example
    $\vectorsym{x}$ is encoded into a vector in the tangent Lie
    algebra of the Lie group manifold formed by Gaussian
    distributions. Then, the vectors is mapped to a UTDAT
    representation of a Gaussian distribution. A latent vector is then
    sampled from this Gaussian distribution and fed to a decoder. The
    whole process is differentiable and optimized using stochastic
    gradient descent.}
  \label{fig:overview}
\end{figure}

\section{Introduction}
\label{sec:introduction}

Unsupervised deep learning is an active research area which shows
considerable progresses recently. Many deep neural network models are
invented to address various problems. For example, auto-encoders (AEs)
\cite{hinton_reducing_2006} are used to learn efficient data codings,
i.e. latent representations. Generative adversarial networks (GANs)
\cite{goodfellow_generative_2014} are powerful on generating
photo-realistic images from latent variables. While having achieved
numerous successes, both AEs and GANs are not without their
disadvantages. On one hand, AEs are good at obtaining a compressed
latent representation of a given input, but hard to generate realistic
samples randomly. On the other hand, GANs are good at randomly
generating realistic samples, but hard to map a given input to its
latent space representation. As a variant of AE, variational
auto-encoders (VAEs) \cite{kingma_auto-encoding_2013} are another kind
of generative models which can also obtain the latent representation
of a given input. The architectures of VAEs are similar to AEs except
that the encoders encode inputs into Gaussian distributions instead of
deterministic vectors. Trained with a Bayesian framework, the decoder
of a VAE is able to generate random samples from latent vectors which
are Gaussian distributed random noises. As a result, many applications
that require manipulating the latent space representations are also
feasible with VAE.

One major problem of VAEs is that the geometric structure of Gaussian
distributions is not considered. Traditional machine learning models
including neural networks as the encoders of VAEs are designed for
vector outputs. However, Gaussian distributions do not form a vector
space. This can be easily shown because the parameter vectors are not
closed under regular vector operators such as vector subtraction. The
variance-covariance matrix must be positive definite but simple vector
subtraction will break this requirement. Naively treating Gaussians as
parameter vectors ignores the geometric structure information of the
space formed by them. To exploit the geometric structural property, we
need to identify what kind of space it is. Gong et al.
\cite{gong_shape_2009} reveals that Gaussians can be represented as a
special kind of affine transformations which are identified as a Lie
group.

In this paper, we view Gaussian distributions from a geometrical
perspective using Lie group theory, and propose a novel generative
model using the encoder-decoder architecture. The overview of our
model is presented in Figure \ref{fig:overview}. As illustrated
therein, the central part of our model is a special Lie group: upper
triangular positive definite affine transform matrices (UTDATs). On
the one hand, UTDATs are matrix representations of Gaussian
distributions. That's to say, there is a one-to-one map between UTDATs
and Gaussian distributions. Therefore, we can analyze the geometric
properties of Gaussian distributions by analyzing the space of
UTDAT. Also, we can sample from Gaussian distributions by
matrix-vector multiplying UTDAT with a standard Gaussian noise
vector. On the other hand, UTDATs form a Lie group. Therefore, one can
work on the tangent spaces (which are Lie algebras) first, then
project back to Lie group by exponential mapping. Since Lie algebras
are vector spaces, they are suitable for most neural network
architectures. As a result, the encoder in our model outputs vectors
in the Lie algebra space. Those vectors are then projected to UTDATs
by a proposed exponential mapping layer. Latent vectors are then
generated by UTDATs and fed to a decoder. Specifically, for Gaussian
distributions with diagonal variance-covariance matrices, we derive a
closed form solution of exponential mapping which is fast and
differentiable. Therefore, our model can be trained by stochastic
gradient descents.

\section{Related works}
\label{sec:related-works}

GANs \cite{goodfellow_generative_2014}
\cite{zhang_self-attention_2018} \cite{miyato_spectral_2018}
\cite{mao_least_2017} are proven effective in generating
photo-realistic images in recent developments of neural
networks. Because of the adversarial training approach, it is
difficult for GANs to map inputs to latent vectors. Although some
approaches \cite{donahue_adversarial_2016}
\cite{schlegl_unsupervised_2017} are proposed to address this problem,
it still remains open and requires further investigation. Compared to
GANs, VAEs \cite{kingma_auto-encoding_2013}
\cite{doersch_tutorial_2016} are generative models which can easily
map an input to its corresponding latent vector. This advantage
enables VAEs to be either used as data compressors or employed in
application scenarios where manipulation of the latent space is
required \cite{yeh_semantic_2016}
\cite{deshpande_learning_2017}. Compared with AEs
\cite{hinton_reducing_2006}, VAEs encode inputs to Gaussian
distributions instead of deterministic latent vectors, and thus enable
them to generate examples. On one hand, Gaussian distributions do not
form a vector space. Naively treating them as vectors will ignore its
geometric properties. On the other hand, most machine learning models
including neural networks are designed to work with vector outputs. To
incorporate the geometric properties of Gaussian distributions, the
type of space of Gaussian distributions needs to be identified first;
then corresponding techniques from geometric theories will be adopted
to design the neural networks.

Geometric theories have been applied to analyze image feature
space. In \cite{tuzel_pedestrian_2008}, covariance matrices are used
as image feature representations for object detection. Because
covariance matrices are symmetric positive definite (SPD) matrices,
which form a Riemannian manifold, a corresponding boosting algorithm
is designed for SPD inputs. In \cite{gong_shape_2009}, Gaussian
distributions are used to model image features and the input space is
analyzed using Lie group theory.

In this paper, we propose a Lie group based generative model using the
\emph{encoder-decoder} architecture. The core of the model is Gaussian
distributions, but we incorporate the geometric properties by working
on the tangent space of Gaussian distributions rather than naively
treating them as vectors.

\section{Gaussians as Lie group}
\label{sec:gaussians-as-lie}

Let $\vectorsym{v}$ be a standard $n$-dimensional Gaussian random
vector
$\vectorsym{v}_0\sim{}\mathcal{N}(\vectorsym{0}, \matrixsym{I})$, then
any new vector
$\vectorsym{v}=\matrixsym{A}\vectorsym{v}_0+\vectorsym{\mu}$ which is
affine transformed from $\vectorsym{v}_0$ is also Gaussian distributed
$\vectorsym{v}\sim{}\mathcal{N}(\vectorsym{\mu}, \matrixsym{\Sigma})$,
where $\matrixsym{\Sigma}=\matrixsym{A}\matrixsym{A}^T$. That is, any
affine transformation can produce a Gaussian distributed random vector
from the standard Gaussian. Furthermore, if we restrict the affine
transformation to be
$\vectorsym{v}=\matrixsym{U}\vectorsym{v}_0+\vectorsym{\mu}$ where
$\matrixsym{U}$ is upper triangular and invertible (i.e. it has
positive eigen values only), then conversely we can find a unique $U$
for any non-degenerate $\matrixsym{\Sigma}$ such that
$\matrixsym{U}\matrixsym{U}^T=\matrixsym{\Sigma}$. In other words,
non-degenerate Gaussian distributions are isomorphic to UTDATs. Let
$G$ denote the matrix form of the following UTDAT:
\begin{equation}
  \matrixsym{G} =
  \begin{bmatrix}
    \matrixsym{U} & \vectorsym{\mu}\\
    0 & 1
  \end{bmatrix},
\end{equation}
then we can identify the type of spaces of Gaussian distributions by
identifying the type of spaces of $G$.

According to Lie theory \cite{knapp_lie_2002}, invertible affine
transformations form a Lie group with matrix multiplication and
inversion as its group operator. It can be easily verified that UTDATs
are closed under matrix multiplication and inversion. So UTDATs form a
subgroup of the general affine group. Since any subgroup of a Lie
group is still a Lie group, UTDATs form a Lie group. In consequence,
Gaussian distributions are elements of a Lie group.

A Lie group is also a differentiable manifold, with the property that
the group operators are compatible with the smooth structure. An
abstract Lie group has many isomorphic instances. Each of them is
called a representation. In Lie theory, matrix representation is a
useful tool for structure analysis. In our case, UTDAT is the matrix
representation of the abstract Lie group formed by Gaussian
distributions.

To exploit the geometric property of Lie group manifolds, the most
important tools are logarithmic mapping, exponential mapping and
geodesic distance. At a specific point of the group manifold, we can
obtain a tangent space which is called Lie algebra in Lie theory. The
Lie group manifold and Lie algebra are analogue to a curve and its
tangent lines in a Euclidean space. Tangent spaces (i.e. Lie algebras)
of a Lie group manifold are vector spaces. In our case, for
$n$-dimensional Gaussians, the corresponding Lie group is
$\frac{1}{2}n(n+3)$ dimensional. Accordingly, its tangent spaces are
$\mathcal{R}^{\frac{1}{2}n(n+3)}$. Note that, at each point of the
group manifold, we have a Lie algebra. We can project a point
$\matrixsym{G}$ of the UTDAT group manifold to the tangent space at a
specific point $\matrixsym{G}_0$ by the logarithmic mapping defined as
\begin{equation}
  \label{eq:logm}
  g=\log(\matrixsym{G}_0^{-1}\matrixsym{G}),
\end{equation}
where the $\log$ operator at the right hand side is matrix logarithm
operator. Note that the points are projected to a vector space even
though the form of the results are still matrices, which means that we
will flatten them to vectors wherever vectors are
required. Specifically, the point $\matrixsym{G}_0$ will be projected
to $\matrixsym{0}$ at its own tangent Lie algebra.

Conversely, the exponential mapping projects points in a tangent space
back to the Lie group manifold. Let $\vectorsym{g}$ be a point in the
tangent space of $\matrixsym{G}_0$, then the exponential mapping is
defined as
\begin{equation}
  \label{eq:expm}
  G=\matrixsym{G}_0\exp(\vectorsym{g}),
\end{equation}
where the $\exp$ operator at the right hand side is matrix exponential
operator. For two points $\matrixsym{G}_1$ and $\matrixsym{G}_2$ of a
Lie group manifold, the geodesic distance is the length of the
shortest path connecting them along the manifold, which is defined as
\begin{equation}
  \label{eq:geodesic}
  d_{LG}(\matrixsym{G}_1, \matrixsym{G}_2)=\lVert\log(\matrixsym{G}_1^{-1}\matrixsym{G}_2)\rVert_{F},
\end{equation}
where $\lVert\cdot\rVert_{F}$ is the Frobenius norm.

\section{Lie group auto-encoder}
\label{sec:lie-group-auto}

\subsection{Overall architecture}
\label{sec:overview}

Suppose we want to generate samples from a complex distribution
$P(\vectorsym{X})$ where $\vectorsym{X}\in\mathcal{R}^D$. One way to
accomplish this task is to generate samples from a joint distribution
$P(\vectorsym{Z}, \vectorsym{X})$ first, then discard the part
belonging to $\vectorsym{Z}$ and keep the part belonging to
$\vectorsym{X}$ only. This seems giving us no benefit at first sight
because it is usually difficult to sample from
$P(\vectorsym{Z}, \vectorsym{X})$ if sampling from $P(\vectorsym{X})$
is hard. However, if we decompose the joint distribution with a
Bayesian formula
\begin{equation}
  \label{eq:bayes}
  P(\vectorsym{Z}, \vectorsym{X}) = P(\vectorsym{X}|\vectorsym{Z})P(\vectorsym{Z}),
\end{equation}
then the joint distribution can be sampled by a two step process:
Firstly sample from $P(\vectorsym{Z})$, then sample from
$P(\vectorsym{X}|\vectorsym{Z})$. The benefits come from the fact that
both $P(\vectorsym{Z})$ and $P(\vectorsym{X|Z})$ may be much easier to
sample from.

Estimating parameters in $P(\vectorsym{Z}, \vectorsym{X})$ as modeled
in Eq. \ref{eq:bayes} is not easy because samples from the joint
distribution are required; however, in most scenarios, we only have
samples $\{\vectorsym{x}_i: i=1,2,\cdots,n\}$ from the marginal
distribution $P(\vectorsym{X})$. To overcome this problem, we augment
each example $\vectorsym{x}_i$ from the marginal distribution to
several examples
$\{(\vectorsym{z}_{ij}, \vectorsym{x}_i): j=1,2,\cdots, m\}$ in the
joint distribution by sampling $\vectorsym{z}_{ij}$ from the
conditional distribution $P(\vectorsym{Z}|\vectorsym{X})$.


Note that $Z$ is an auxiliary random vector helping us perform
sampling from the marginal distribution $P(\vectorsym{X})$, so it can
be any kind of distribution but should be easy to sample from. In this
paper, we let
$P(\vectorsym{Z})\sim\mathcal{N}(\vectorsym{0}, \matrixsym{I})$.

In practice, $P(\vectorsym{Z}|\vectorsym{X})$ should be chosen
according to the type of data space of $\vectorsym{X}$. For example,
if $X$ is continuous, we can model $P(\vectorsym{X}|\vectorsym{Z})$ as
a Gaussian distribution with a fixed isotropic variance-covariance
matrix. For binary and categorical typed $\vectorsym{X}$, Bernoulli
and multinomial distributions can be used, respectively.

Given $P(\vectorsym{Z})$ and $P(\vectorsym{X}|\vectorsym{Z})$,
$P(\vectorsym{Z}|\vectorsym{X})$ is usually complex and thus difficult
to sample from. So we sample $\vectorsym{z}_{ij}$ from another
distribution $Q(\vectorsym{Z}|\vectorsym{X})$ instead. In this paper,
we model $Q(\vectorsym{Z}|\vectorsym{X})$ as Gaussian distributions
with diagonal variance-covariance
matrices. $Q(\vectorsym{Z}|\vectorsym{X})$ should satisfy the
following objectives as much as possible:
\begin{itemize}
\item $Q(\vectorsym{Z}|\vectorsym{X})$ should approximate
  $P(\vectorsym{Z}|\vectorsym{X})$. Therefore, given
  $\vectorsym{z}_{ij}$ sampled from
  $Q(\vectorsym{Z}|\vectorsym{x}_i)$, $\hat{\vectorsym{x}}_i$ sampled
  from $P(\vectorsym{X}|\vectorsym{Z})$ should reconstruct
  $\vectorsym{x}_i$.
\item $\{\vectorsym{z}_{ij}:i=1,2,\cdots,n; j=1,2,\cdots,m\}$ should
    fit the marginal $P(\vectorsym{Z})$ well.
\end{itemize}

To optimize the first objective, we minimize the reconstruction loss
$L_{rec}$, which is the mean squared error (MSE) for continuous
$\vectorsym{X}$ and cross-entropy for binary and categorical
$\vectorsym{X}$.

For the second objective, directly optimizing it using
$\vectorsym{z}_{ij}$ is not practical because we need a large sample
size $m$ for $P(\matrixsym{Z}|\matrixsym{X})$ to accurately estimate
model parameters.  The total sample size of $\vectorsym{z}_{ij}$,
which is $mn$, is too big for computation. To overcome this problem,
we consider Gaussian distributions as points in the corresponding Lie
group as we discussed in section \ref{sec:gaussians-as-lie}. Note that
the set $\{\vectorsym{z}_{ij}, i=1,2,\cdots,n; j=1,2,\cdots,m\}$ is
sampled from a set of Gaussian distributions
$Q(\vectorsym{Z}|\vectorsym{x}_i)$. The second objective implies the
average distribution of those Gaussians should be $P(\vectorsym{Z})$,
which is a standard Gaussian. However, Gaussian distributions, which
are equivalently represented as UTDATs, do not conform to the commonly
used Euclidean geometry. Instead, we need to find the intrinsic mean
of those Gaussians through Lie group geometry because Gaussian
distributions have a Lie group structure. We derive a Lie group
intrinsic loss $L_{LG}$ to optimize the second objective. The details
of $L_{LG}$ will be present in subsection
\ref{sec:lie-group-intrinsic}.

In our proposed Lie group auto-encoder (LGAE),
$P(\vectorsym{X}|\vectorsym{Z})$ is called a \emph{decoder} or
\emph{generator}, and is implemented with neural
networks. $Q(\vectorsym{Z}|\vectorsym{X})$ is also implemented with
neural networks. Note that $Q(\vectorsym{Z}|\vectorsym{X})$ is a
Gaussian distribution, so the corresponding neural network is a
function whose output is a Gaussian distribution. Neural networks as
well as many other machine learning models are typically designed for
vector outputs. Being intrinsically a Lie group as discussed in
section \ref{sec:gaussians-as-lie}, Gaussian distributions do not form
a vector space. To best exploit the geometric structure of the
Gaussians, we first estimate corresponding points $\matrixsym{g}_i$ in
the tangent Lie algebra at the position of the intrinsic mean of
$\{\matrixsym{G}_i, i=1,2,\cdots,n\}$ using neural networks. As
$L_{LG}$ requires the intrinsic mean to be the standard Gaussian
$P(\vectorsym{Z})=\mathcal{N}(\vectorsym{0}, \matrixsym{I})$, whose
UTDAT representation is the identity matrix $\matrixsym{I}$, the
corresponding point $\matrixsym{g}_i$ in the tangent space of
$\matrixsym{G}_i$ is
\begin{equation}
  \label{eq:gi}
  \matrixsym{g}_i=\log(\matrixsym{G}_i).
\end{equation}
Since $\{\matrixsym{g}_i, i=1,2,\cdots, n\}$ are in a vector space,
they can be well estimated by neural networks. $\matrixsym{g}_i$s are
then projected to the Lie group by an exponential mapping layer
\begin{equation}
  \label{eq:Gi}
  \matrixsym{G}_i=\exp(\matrixsym{g}_i).
\end{equation}
For diagonal Gaussians, we derive a closed-form solution of the
exponential mapping which eliminates the requirement of matrix
exponential operator. The details will be presented in subsection
\ref{sec:expon-mapp-layer}.

The whole architecture of LGAE is summarized in Figure
\ref{fig:overview}. A typical forward procedure works as follows:
Firstly, the \emph{encoder} encodes an input $\vectorsym{x}_i$ into a
point $\matrixsym{g}_i$ in the tangent Lie algebra. The exponential
mapping layer then projects $\matrixsym{g}_i$ to the UTDAT matrix
$\matrixsym{G}_i$ of the Lie group manifold. A latent vector
$\vectorsym{z}$ is then sampled from the Gaussian distribution
represented by $\matrixsym{G}_i$ by multiplying $\matrixsym{G}_i$ with
a standard Gaussian noise vector. The details of the sampling
operation will be described in section \ref{sec:sampl-from-gauss}. The
\emph{decoder} (or \emph{generator}) network then generates
$\hat{\vectorsym{x}}_i$ which is the reconstructed version of
$\vectorsym{x}_i$. The whole network is optimized by minimizing the
following loss
\begin{equation}
  \label{eq:loss}
  L = \lambda{}L_{LG}+L_{rec},
\end{equation}
where $L_{LG}$ and $L_{rec}$ are the Lie group intrinsic loss and
reconstruction loss, respectively. Because the whole forward process
and the loss are differentiable, the optimization can be achieved by
\emph{stochastic gradient descent} method.

\subsection{Exponential mapping layer}
\label{sec:expon-mapp-layer}

We derive the exponential mapping
$\matrixsym{g}_i=\exp({\matrixsym{G}_i})$ for diagonal Gaussians. When
$\matrixsym{G}_i\sim\mathcal{N}(\vectorsym{\mu}_i,\matrixsym{\Sigma}_i)$
is diagonal, we have
\begin{equation}
  \label{eq:cov}
  \matrixsym{\Sigma}_i=
  \begin{bmatrix}
    \sigma_{i1}^2 & &\\
    & \sigma_{i2}^2 & \\
    & & \ddots &\\
    & & & \sigma_{iK}^2.
  \end{bmatrix}.
\end{equation}
The following theorem gives the forms of $\matrixsym{G}_i$ and
$\matrixsym{g}_i$, as well as their relationship.
\begin{mythm}
  Let $\matrixsym{G}_i$ be the UTDAT and $\matrixsym{g}_i$ be the
  corresponding vector in its tangent Lie algebra at the standard
  Gaussian. Then
  \begin{align}
    \label{eq:diagonal-Gi}
    \matrixsym{G}_i&=
    \begin{bmatrix}
      \sigma_{i1} & & &\mu_{i1}\\
      & \sigma_{i2} & & \mu_{i2}\\
      & & \ddots & \vdots\\
      0 & 0 & \hdots &   1
    \end{bmatrix}\\
    \label{eq:diagonal-gi}
    \matrixsym{g}_i&=
    \begin{bmatrix}
      \phi_{i1} & & &\theta_{i1}\\
      & \phi_{i2} & & \theta_{i2}\\
      & & \ddots & \vdots\\
      0 & 0 & \hdots &   1
    \end{bmatrix},
  \end{align}
  where
  \begin{align}
    \phi_{ik} &=\log(\sigma_{ik})\\
    \theta_{ik}&=\frac{\mu_{ik}\log(\sigma_{ik})}{\sigma_{ik}-1}.
  \end{align}
\end{mythm}

\begin{proof}
  By the definition of UTDAT, we can straightforwardly get
  Eq. \ref{eq:diagonal-Gi}. Let
  $\matrixsym{H}=\matrixsym{G}_i-\matrixsym{I}$. Using the series form
  of matrix logarithm, we have
  \begin{align}
  \matrixsym{g}_i
  &=\log(\matrixsym{G}_i) \nonumber \\
    &=\log(\matrixsym{I}+\matrixsym{H}) \nonumber \\
    \label{eq:diagonal-log-last}
  &=\sum_{t=1}^{\infty}(-1)^{t-1}\frac{\matrixsym{H}^t}{t}.
  \end{align}
  By substituting $\matrixsym{H}$ into \ref{eq:diagonal-log-last}, we get Eq. \ref{eq:diagonal-gi} and the following:
  \begin{align*}
  \phi_{ik}
  &=\sum\limits_{t=1}^{\infty}(-1)^{t-1}\frac{(\sigma_{ik}-1)^t}{t}\\
  &=\log(\sigma_{ik})
\end{align*}
and
\begin{align*}
  \theta_{ik}
  &=\sum\limits_{t=1}^{\infty}(-1)^{t-1}\frac{\mu_{ik}(\sigma_{ik}-1)^{t-1}}{t}\\
  &=\frac{\mu_{ik}\log(\sigma_{ik})}{\sigma_{ik}-1}.
\end{align*}
Alternatively, after we identify $\matrixsym{g}_i$ has the form as in
Eq. \ref{eq:diagonal-gi}, we can derive the exponential mapping by the
definition of matrix exponential
\begin{align*}
  \matrixsym{G}_i&=\exp(\matrixsym{g}_i)=\sum_{t=0}^\infty\frac{\matrixsym{g}_i^t}{t!}\\
  &=\begin{bmatrix}
    \sum_{t=0}^\infty\frac{\phi_{i1}^t}{t!} &  & \theta_{i1}\sum_{t=1}^\infty\frac{\phi_{i1}^{t-1}}{t!}\\
    & \ddots & \vdots\\
    0 & \hdots & 1
  \end{bmatrix}\\
 &=\begin{bmatrix}
    e^{\phi_{i1}} &  & \frac{\theta_{i1}}{\phi_{i1}}\left(\sum_{t=0}^\infty\frac{\phi_{i1}^t}{t!}-1\right)\\
    & \ddots & \vdots\\
    0 & \hdots & 1
 \end{bmatrix}\\
 &=\begin{bmatrix}
   e^{\phi_{i1}} &  & \frac{\theta_{ik}(e^{\phi_{i1}}-1)}{\phi_{i1}}\\
    & \ddots & \vdots\\
    0 & \hdots & 1
 \end{bmatrix}.
\end{align*}
\end{proof}

The exponential mapping layer is expressed as 
\begin{align}
  \label{eq:exp-map}
  \sigma_{ik}&=e^{\phi_{ik}}\\
  \mu_{ik}&=\frac{\theta_{ik}(e^{\phi_{ik}}-1)}{\phi_{ik}}
\end{align}
Note that if $\sigma_{ik}=1$ (i.e. $\phi_{ik}=0$), then
$\mu_{ik}=\theta_{ik}$ due to the fact that
$\lim_{x\to{}0}\frac{\log(x+1)}{x}=1$ or
$\lim_{x\to{}0}\frac{e^x-1}{x}=1$.

\subsection{Lie group intrinsic loss}
\label{sec:lie-group-intrinsic}

Let $\matrixsym{G}_i$ be the UTDAT
representation of $P(\vectorsym{Z}|\vectorsym{x}_i)$. The intrinsic
mean $\matrixsym{G}^*$ of those $\matrixsym{G}_i$s is defined as
\begin{equation}
  \label{eq:intrinsic-mean}
  \matrixsym{G}^*=\underset{\matrixsym{G}}{\arg\min}\sum_{i=1}^nd_{LG}^2(\matrixsym{G}, \matrixsym{G}_i).
\end{equation}
The second objective in the previous subsection requires that
$\matrixsym{G}^*=\matrixsym{I}$, which is equivalent to minimizing the
loss
\begin{align}
  L_{LG}
  \label{eq:intrinsic-mean-1}
  &=\sum_{i=1}^nd_{LG}^2(\matrixsym{I}, \matrixsym{G}_i)\\
  &=\sum_{i=1}^n\lVert\log(\matrixsym{G}_i)\rVert_{F}^2 \nonumber \\
  \label{eq:intrinsic-mean-2}
  &=\sum_{i=1}^n\lVert\matrixsym{g}_i\rVert_{F}^2.
\end{align}
So the intrinsic loss plays a role of regularization during the
training. Since the tangent Lie algebra is a vector space, the
Frobenius norm is equivalent to the $l^2$-norm if we flatten matrix
$\matrixsym{g}_i$ to a vector. Eq. \ref{eq:intrinsic-mean-1} plays a
role of regularization which requires all the Gaussians
$\matrixsym{G}_i$ to be grouped together around the standard
Gaussian. Eq. \ref{eq:intrinsic-mean-2} shows that we can regularize
on the tangent Lie algebra instead, which avoids the matrix logarithm
operation. Specifically, for diagonal Gaussians, we have
\begin{align}
  L_{LG}
  &=\sum_{i=1}^n\lVert\matrixsym{g}_i\rVert_{F}^2\\
  &=\sum_{i=1}^n\sum_{k=1}^K(\phi_{ik}^2+\theta_{ik}^2)
\end{align}

\subsection{Sampling from Gaussians}
\label{sec:sampl-from-gauss}

According to the properties of Gaussian distributions discussed in
section \ref{sec:gaussians-as-lie}, sampling from an arbitrary
Gaussian distribution can be achieved by transforming a standard
Gaussian distribution with the corresponding, i.e.
\begin{equation}
  \label{eq:sample}
  \begin{bmatrix}
    \vectorsym{z}_{ij}\\
    1
  \end{bmatrix}=\matrixsym{G}_i
  \begin{bmatrix}
    \vectorsym{v}_{ij}\\
    1
  \end{bmatrix},
\end{equation}
where $\vectorsym{v}_{ij}$ is sample from
$\vectorsym{V}\sim\mathcal{N}(\vectorsym{0}, \matrixsym{I})$. Note
that, this sampling operator is differentiable, which means that
gradients can be back-propagated through the sampling layer to the
previous layers. When $\matrixsym{G}_i$ is a diagonal Gaussian, we
have
\begin{equation}
  \label{eq:diag-sample}
  \vectorsym{z}_{ij}=\vectorsym{\sigma}\odot\vectorsym{v}_{ij}+\vectorsym{\mu},
\end{equation}
where $\vectorsym{\sigma}=[\sigma_{i1},\cdots,\sigma_{iK}]^T$,
$\vectorsym{\mu}=[\mu_{i1},\cdots,\mu_{iK}]^T$ and $\odot$ is the
element-wise multiplication. Therefore, the re-parameterization trick
in \cite{kingma_auto-encoding_2013} is a special case of sampling of
UTDAT represented Gaussian distributions.

\begin{figure*}[ht]
  \centering
  \begin{tabular}{cccc}
    \includegraphics[scale=0.22]{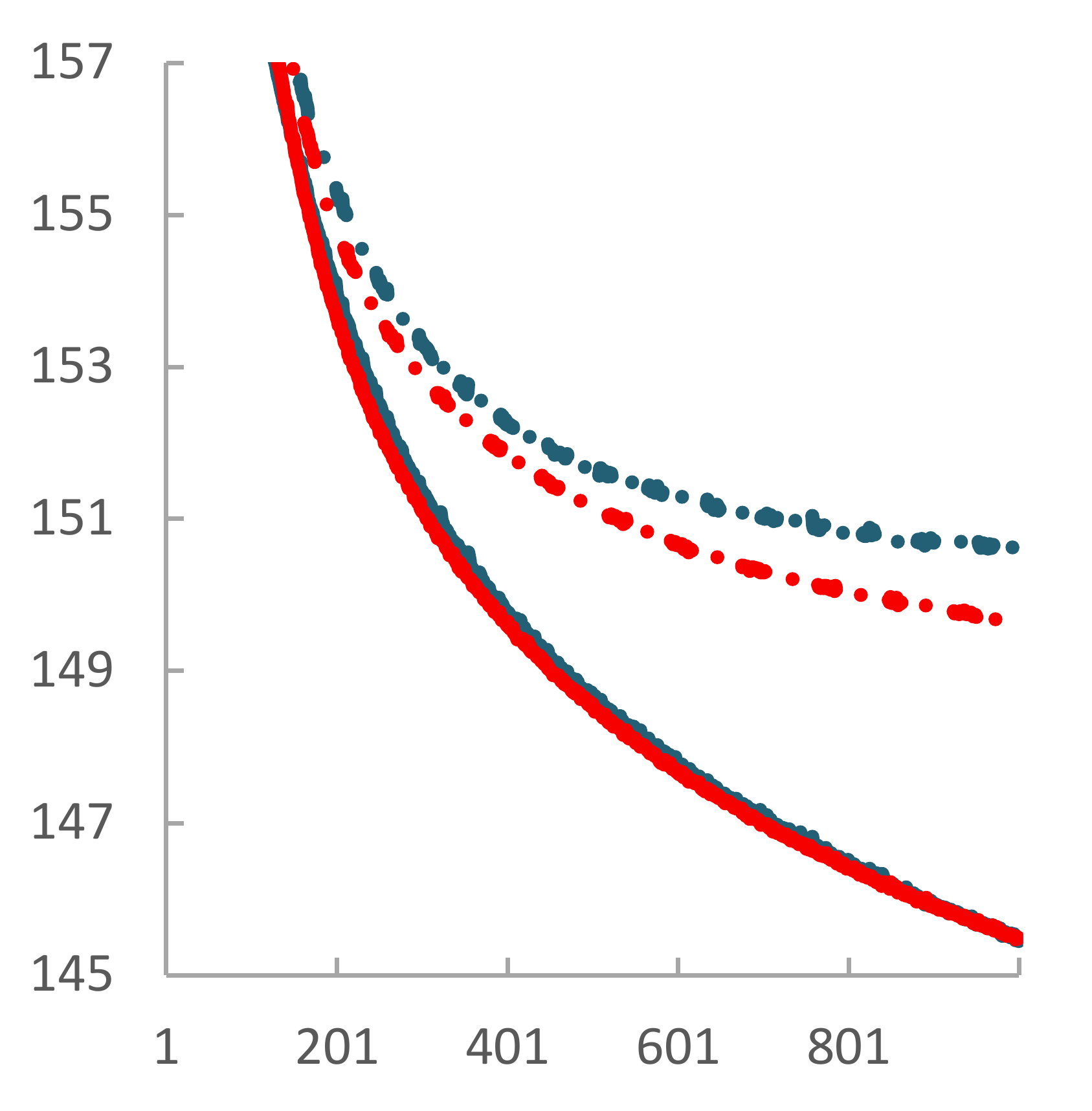}
    & \includegraphics[scale=0.22]{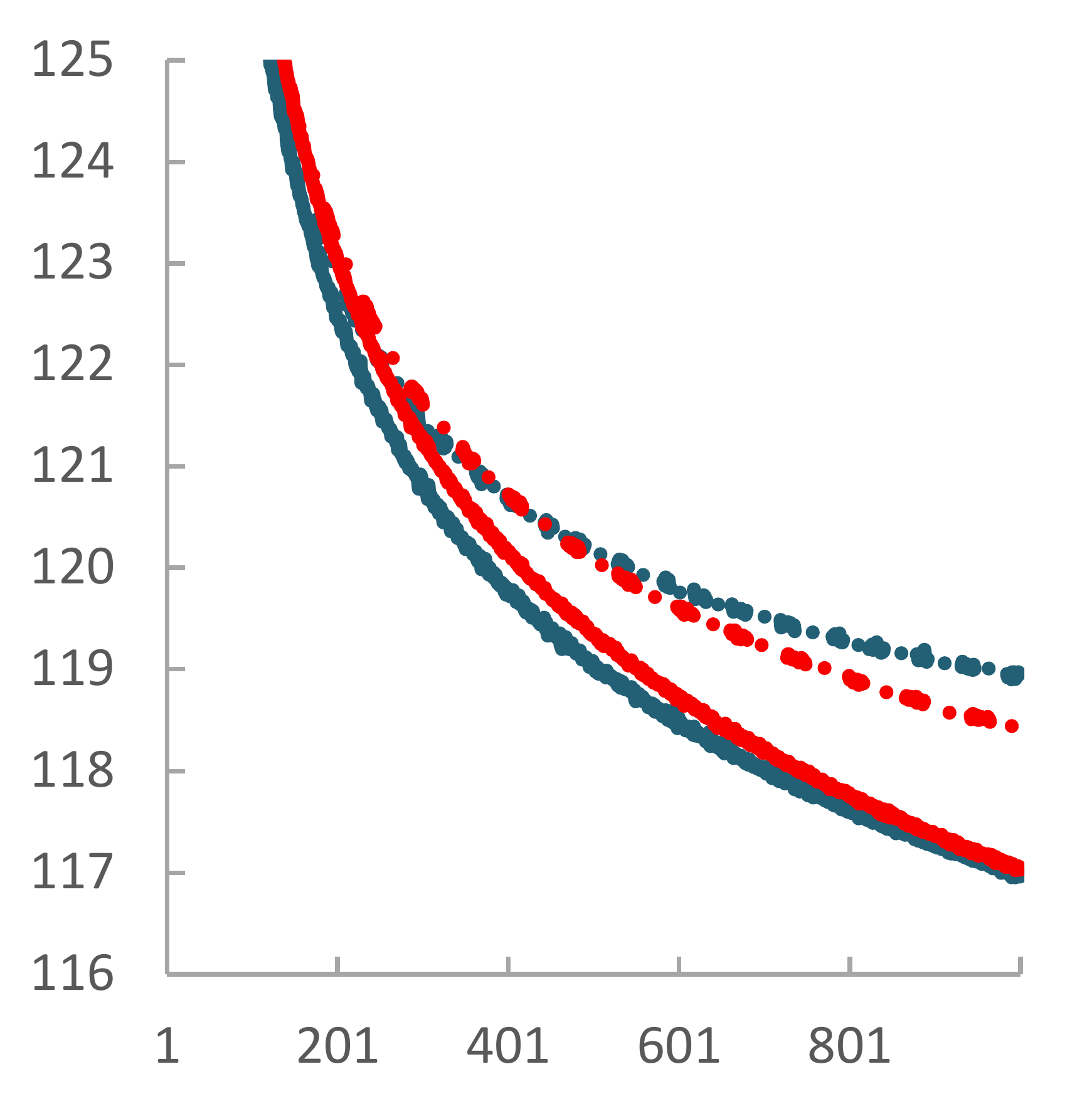}
    & \includegraphics[scale=0.22]{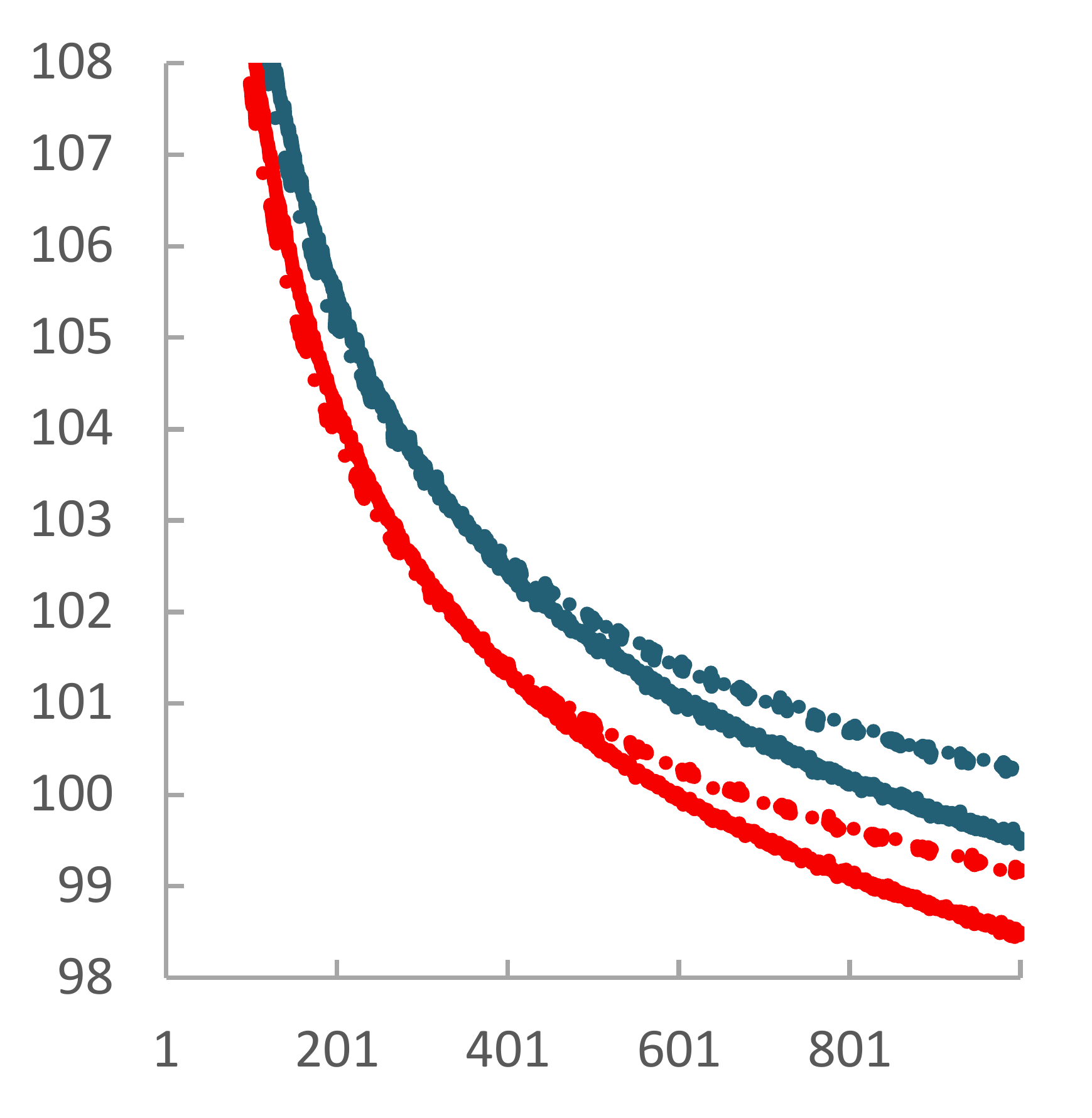}
    & \includegraphics[scale=0.22]{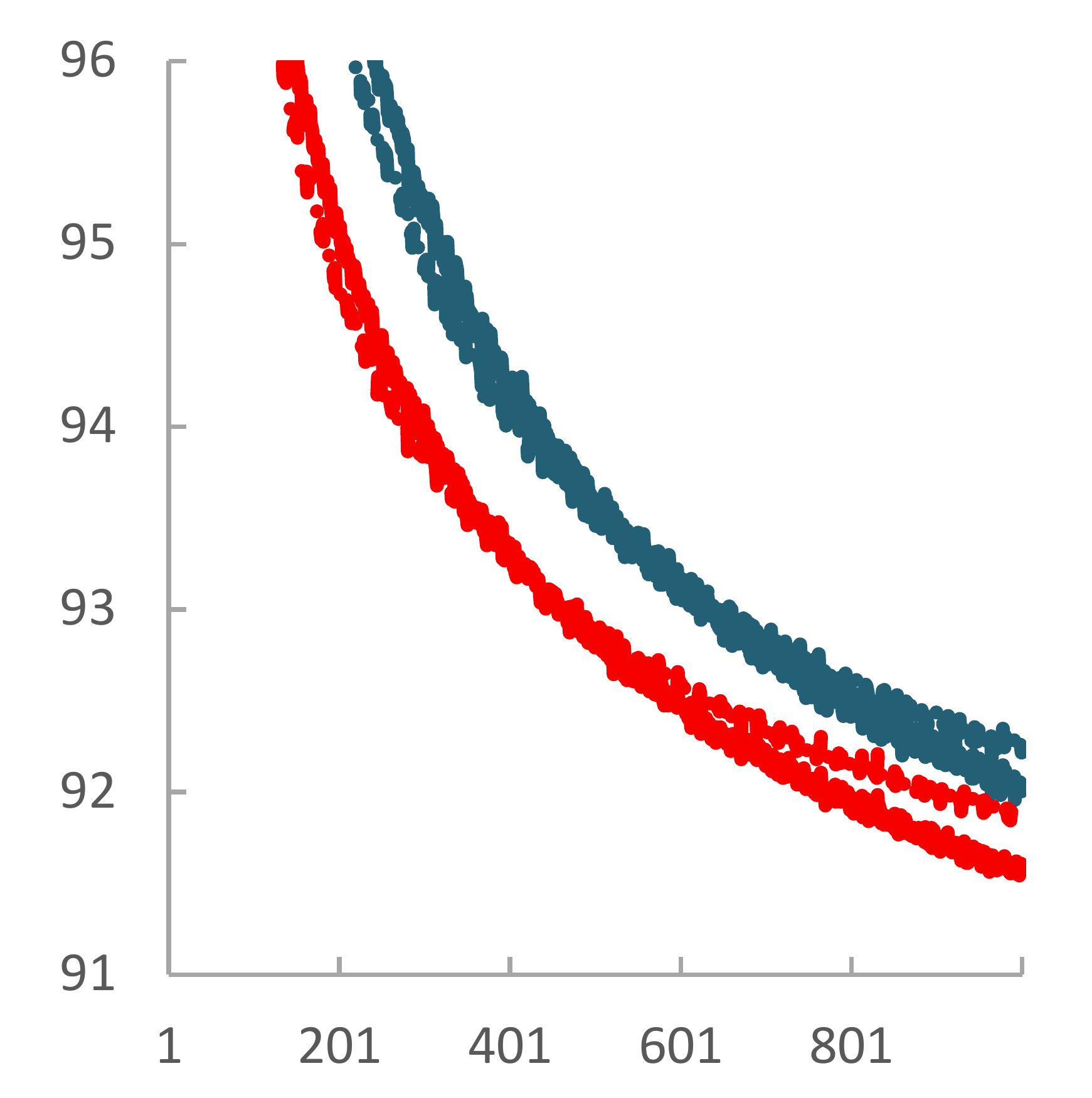}\\
    MNIST (K=2) & MNIST (K=5) & MNIST (K=10) & MNIST (K=20)\\
    & & & \\
    \includegraphics[scale=0.15]{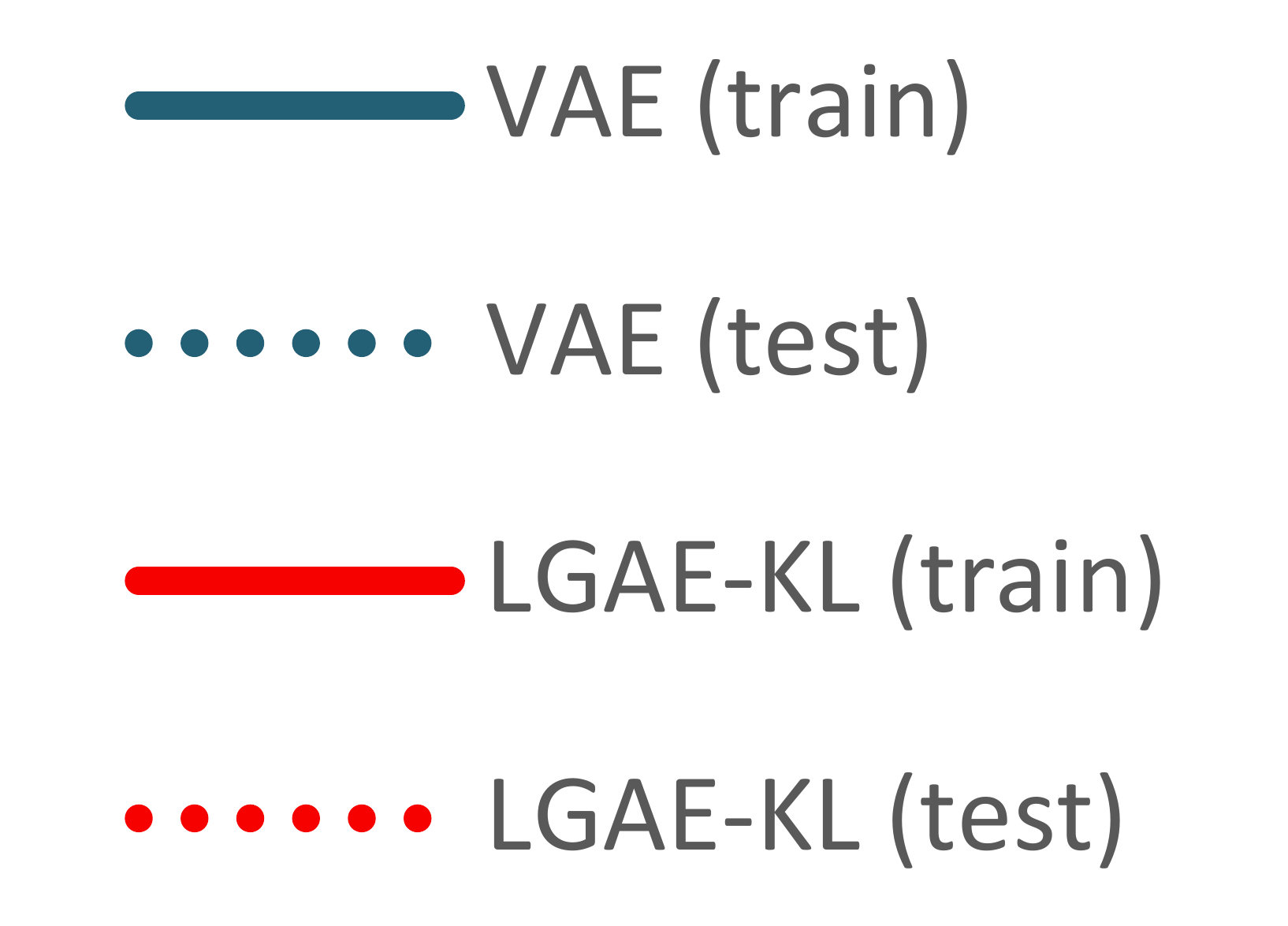}
    & \includegraphics[scale=0.22]{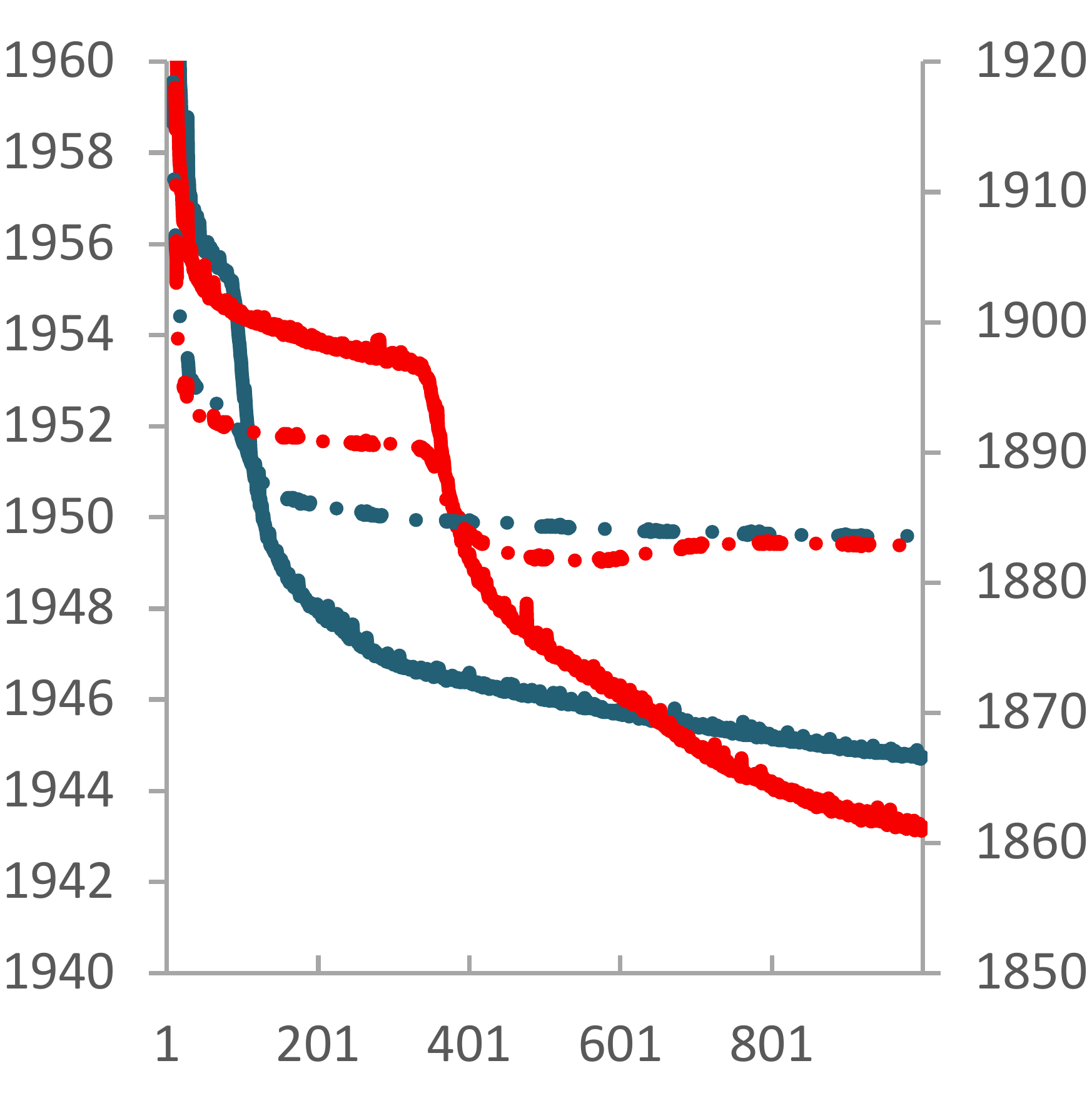}
    & \includegraphics[scale=0.22]{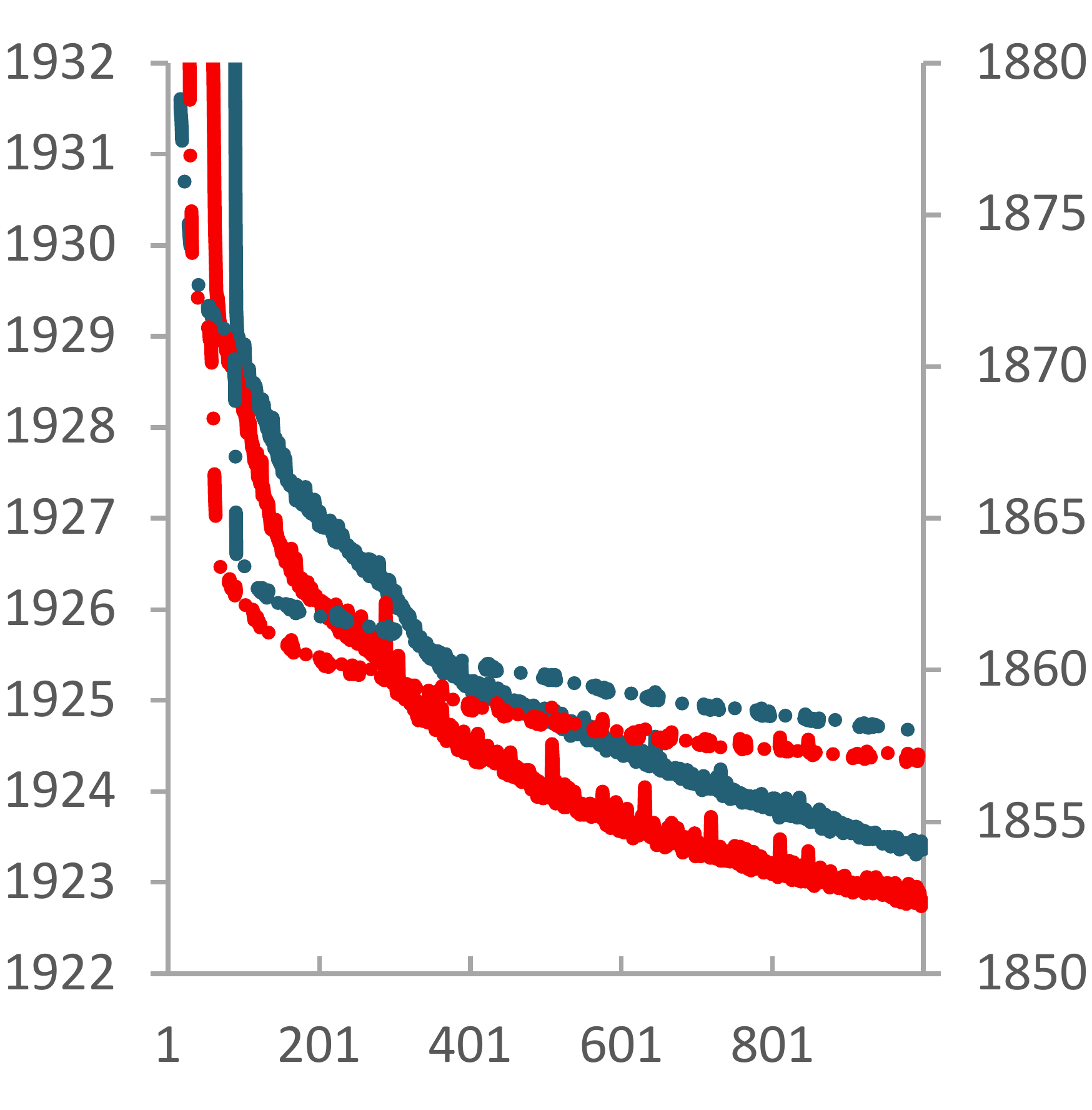}
    & \includegraphics[scale=0.22]{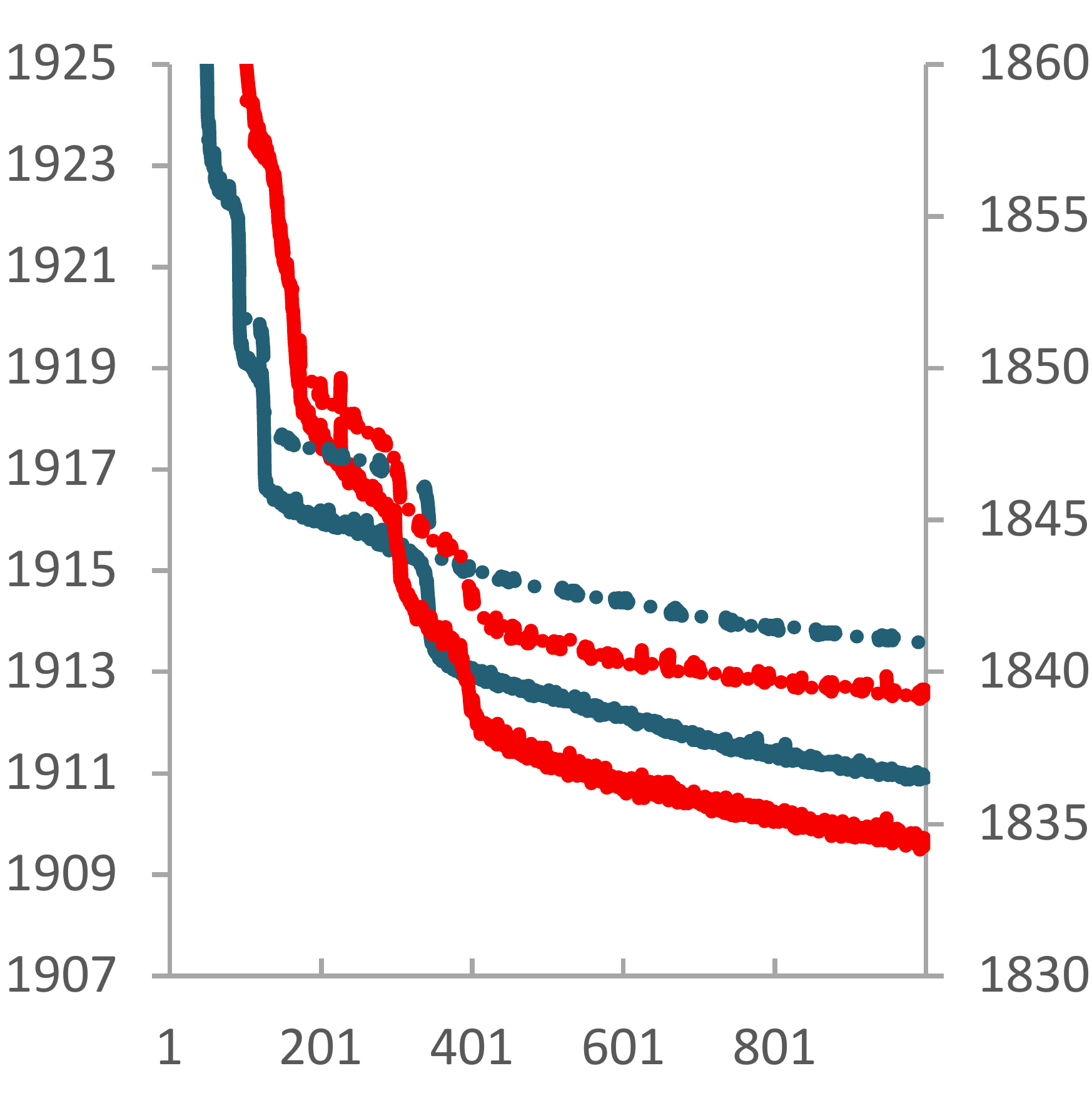}\\
    & SVHN (K=2) & SVHN (K=5) & SVHN (K=10)
    
  \end{tabular}
  \caption{Loss values (i.e. negative lower bound of marginal
    likelihood as described in \cite{kingma_auto-encoding_2013})
    versus training progresses. The horizontal and vertical axes are
    the number of epochs trained and the loss values,
    respectively. Note that loss values on the SVHN dataset are
    plotted with two different vertical axes to avoid the scaling
    problem caused by the gap between training and testing loss
    values. The left and right axes are for traing and testing sets,
    respectively.}
  \label{fig:loss}
\end{figure*}

\section{Discussion}
\label{sec:discussion}

Although our proposed LGAE and VAE both
\cite{kingma_auto-encoding_2013} have an \emph{encoder-decoder} based
architecture, they are essentially different. The loss function of
VAE, which is
\begin{equation}
  \label{eq:loss-VAE}
  L=\operatorname{KL}(Q(\vectorsym{Z}|\vectorsym{X})\rVert\mathcal{N}(\vectorsym{0},\matrixsym{I}))+L_{rec},
\end{equation}
is derived from the Bayesian lower bound of the marginal likelihood of
the data. In contrast, the loss function of LGAE is derived from a
geometrical perspective. Further, the Lie group intrinsic loss
$L_{LG}$ in Eq. \eqref{eq:loss} is a real metric, but the
KL-divergence in Eq. \eqref{eq:loss-VAE} is not. For examples, the
KL-divergence is not symmetric, nor does it satisfy the triangle
inequality.

Further, while both LGAE and VAE estimate Gaussian distributions using
neural networks, VAE does not address the non-vector output
problem. As a contrast, we systematically address this problem and
design an exponential mapping layer to solve it. One requirement
arising from the non-vector property of Gaussian distributions is that
the variance parameters be positive. To satisfy this requirement,
\cite{kingma_auto-encoding_2013} estimate the logarithm of variance
instead. This technique is equivalent to performing the exponential
mapping for the variance part. Without a theoretical foundation, it
was trial-and-error to choose \emph{exp} over other activations such
as \emph{relu} and \emph{softplus}. Our theoretical results confirms
that \emph{exp} makes more sense than others. Moreover, our
theoretical results further show that a better way is to consider a
Gaussian distribution as a whole rather than treat its variance part
only and address the problem in an empirical way.

Because the points of the tangent Lie algebra are already vectors, we
propose to use them as compressed representations of the input data
examples. These vectors contain information of the Gaussian
distributions and already incorporate the Lie group structural
information of the Gaussian distributions; therefore, they are more
informative than either a single mean vector or concatenating the mean
vector and variance vector naively together.

\begin{figure*}[ht]
  \centering
  \begin{tabular}{cccc}
    \includegraphics[scale=0.28]{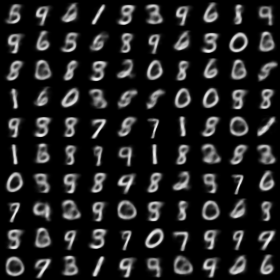}
    & \includegraphics[scale=0.28]{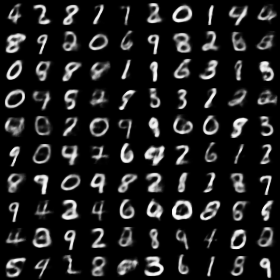}
    & \includegraphics[scale=0.28]{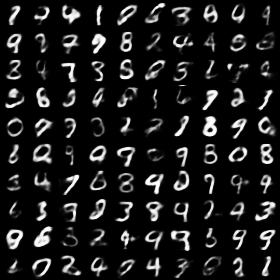}
    & \includegraphics[scale=0.28]{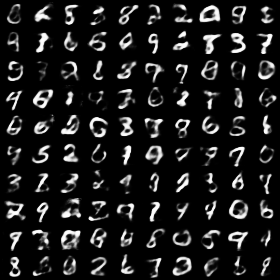}\\
    VAE (K=2) & VAE (K=5) & VAE (K=10) & VAE (K=20)\\
    & & & \\
    \includegraphics[scale=0.28]{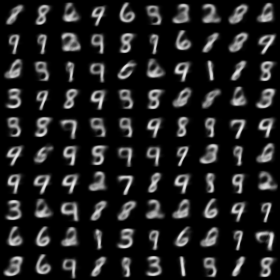}
    & \includegraphics[scale=0.28]{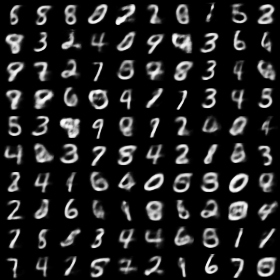}
    & \includegraphics[scale=0.28]{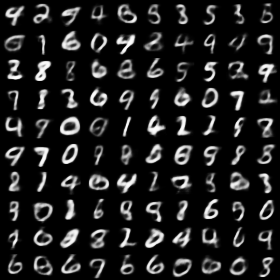}
    & \includegraphics[scale=0.28]{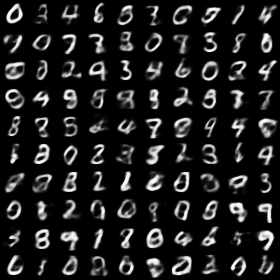}\\
    LGAE (K=2) & LGAE (K=5) & LGAE (K=10) & LGAE (K=20)\\
  \end{tabular}
  \caption{Generated images from randomly sampled latent vectors for
    the MNIST dataset. The upper and lower rows are generated by the VAE
    and LGAE models, respectively.}
  \label{fig:latent-mnist}
\end{figure*}

\section{Experiments}
\label{sec:experiments}

\subsection{Datasets}
\label{sec:datasets}

The proposed LGAE model is evaluated on two benchmark datasets:
\begin{itemize}
\item \textbf{MNIST}: The MNIST dataset
  \cite{lecun_gradient-based_1998-1} consists of a training set of
  $60,000$ examples of handwritten digits, and a test set of $10,000$
  examples. The digits have been size-normalized and centered into
  fixed-size $28\times{}28$ images.
\item \textbf{SVHN}: The SVHN dataset \cite{netzer_reading_2011} is
  also a collection of images of digits. But the background of image
  is more clutter than MNIST, so it is significantly harder to
  classify. It includes a training set of $73,257$ examples, a test
  set of $26,032$ examples, and an extra training set of $531,131$
  examples. In our experiments, we use the training and test sets
  only, and the extra training set is not used throughout the
  experiments.
\end{itemize}

\subsection{Settings}
\label{sec:settings}

Since VAE \cite{kingma_auto-encoding_2013} is the most related model
of LGAE, we use VAE as a baseline for comparisons. We follow the exact
experimental settings of \cite{kingma_auto-encoding_2013}. That is,
MLP with $500$ hidden units are used as encoder and decoder. In each
hidden layer, non-linear activation \emph{tanh} are applied. The
parameters of neurons are initialized by random sampling from
$\mathcal{N}(0, 0.01)$ and are optimized by Adagrad
\cite{duchi_adaptive_2011} with learning rate $0.01$. Mini-batches of
size $100$ are used. For the LGAE model, we use $\lambda=0.5$ as the
weight for the Lie group intrinsic loss.

For both the MNIST and SVHN datasets, we normalize the pixel values of
the images to be in the range $[0,1]$. Cross entropies between those
true pixel values and predicted values are used as reconstruction
errors in Eq. \eqref{eq:loss}. Both the VAE and LGAE are implemented
with PyTorch \cite{paszke_automatic_2017}. Note that there is no
matrix operation in the LGAE implementation thanks to the element-wise
closed-form solution presented in Section \ref{sec:expon-mapp-layer}
and \ref{sec:lie-group-intrinsic}. Therefore, the run-time is almost
the same as VAE. On a Nvidia GeForce GTX 1080 graphic card, it takes
about $12.5$ and $25$ seconds to train on the training set and test on
both the training and test sets for one epoch with mini-batches of
size $100$.

\begin{figure*}[t]
  \centering
  \begin{tabular}{ccc}
    \includegraphics[scale=0.28]{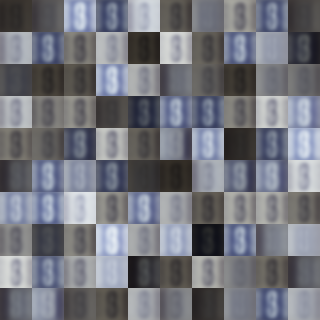}
    & \includegraphics[scale=0.28]{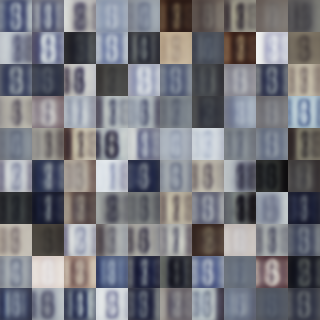}
    & \includegraphics[scale=0.28]{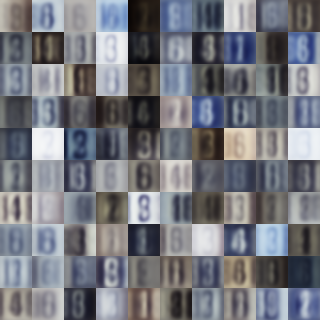}\\
    VAE (K=2) & VAE (K=5) & VAE (K=10)\\
    & &\\
    \includegraphics[scale=0.28]{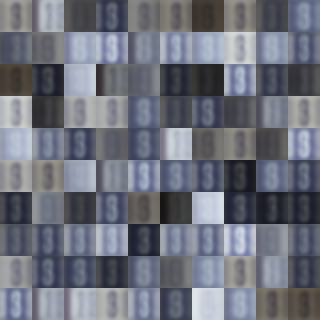}
    & \includegraphics[scale=0.28]{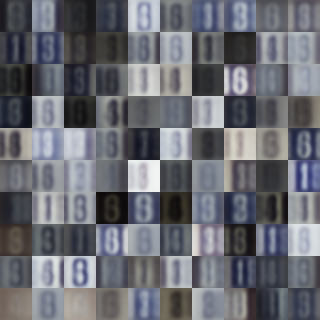}
    & \includegraphics[scale=0.28]{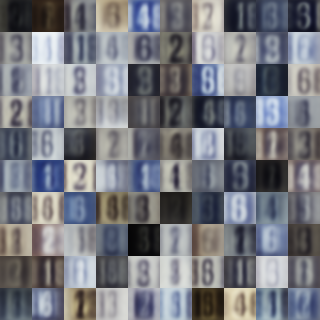}\\
    LGAE (K=2) & LGAE (K=5) & LGAE (K=10)\\
  \end{tabular}
  \caption{Generated images from randomly sampled latent vectors for
    the MNIST dataset. The upper and lower rows are generated by the
    VAE and LGAE models, respectively.}
  \label{fig:latent-svhn}
\end{figure*}

\subsection{Results}
\label{sec:results}

In the first experiment, we investigate the effectiveness of the
proposed exponential mapping layer. We design a variant of LGAE which
uses the same loss as VAE; i.e., we replace the Lie group intrinsic
loss with KL divergence but keep the exponential mapping layer in the
model. We call this variant LGAE-KL. Because it has the same loss
formula as VAE, it is fair to compare their values during training. We
train VAE and LGAE-KL on the training sets of MNIST and SVHN for
$1,000$ epochs. It is shown in \cite{kingma_auto-encoding_2013} that
the negative value of this loss is the lower bound of the marginal
likelihood in a Bayesian perspective. After each epoch, we evaluate
the loss both on the training and test sets. The values are plotted in
Figure \ref{fig:loss}. The curves show that our LGAE-KL obtains
smaller values of the loss both on training and test sets, which
indicates that learning on the tangent Lie algebra is more effective
than ignoring the Lie group structure of the Gaussian
distributions. Note that LGAE also obtains smaller lower bounds
consistently with different values of $K$. On the SVHN dataset, there
is a large gap between training and testing loss values, which is
possibly caused by the different sizes of the two sets, as well as the
cluttered backgrounds in images of the dataset. A simple multi-layer
perceptrons (MLP) network with two layers may not be able to
sufficiently fit the data; therefore, the lower bound value on the
training set decreases too slowly to close the gap. In the future,
designing a deep convolutional network architecture for the
\emph{encoder} and \emph{decoder} is a promising direction to extend
LGAE model to handle natural images.

In the second experiment, we compare the encoding capability of LGAE
with VAE. We train LGAE, LGAE-KL and VAE on the training sets. Then we
obtain the encoded representations of examples from both the training
and test sets. A nearest centroid classifier is then trained on the
representations of the training set and tested on the test set. Since
the nearest centroid is a very simple classifier which has no
hyper-parameters, the classification accuracy on the test set
indicates the representation power of encoded inputs. For VAE we test
two different representations: mean vector and a concatenation of mean
vector and covariance matrix (i.e. concatenate parameters of Gaussian
distributions). For LGAE-KL and LGAE we test those two kinds of
representations as well as the Lie algebra vector. Table
\ref{tab:nc-mnist} summarizes the results on the MNIST dataset.
\begin{table}[ht]
  \centering
  \caption{Classification accuracies ($\%$) on the MNIST dataset.}
  \begin{tabular}{l|c|c|c|c}
    \toprule
    Method & $K=2$ & $K=5$ & $K=10$ \\
    \midrule
    \midrule
    VAE ($\vectorsym{\mu}$) & $56.22$ & $78.53$ & $85.68$\\
    VAE ($\vectorsym{\mu} \Vert \matrixsym{\Sigma}$) & $56.22$ & $78.53$ & $85.68$\\
    \midrule
    \midrule
    LGAE-KL ($\vectorsym{\mu}$) & $58.88$ & $80.77$ & $85.86$\\
    LGAE-KL ($\vectorsym{\mu} \Vert \matrixsym{\Sigma}$) & $58.88$ & $80.77$ & $85.90$\\
    LGAE-KL ($\matrixsym{g}$) & $60.82$ & $80.54$ & $86.20$\\
    LGAE ($\vectorsym{\mu}$) & $59.37$ & $79.81$ & $85.45$\\
    LGAE ($\vectorsym{\mu} \Vert \matrixsym{\Sigma}$) & $54.80$ & $79.26$ & $85.47$\\
    LGAE ($\matrixsym{g}$) & $\bm{68.70}$ & $\bm{83.27}$ & $\bm{86.88}$\\
    \bottomrule
  \end{tabular}
  \label{tab:nc-mnist}
\end{table}
From Table \ref{tab:nc-mnist}, we can see that LGAE-KL gets better
performance than VAE, which reconfirms the effectiveness of the
proposed exponential mapping layer. LGAE performs the best, which
indicates the superiority of the Lie group intrinsic loss over KL
divergence. Moreover, the results also show that naively concatenating
covariance with mean does not contribute much to the performances, and
sometimes even hurts it. This phenomenon indicates that treating
Gaussians as vectors cannot fully extract important geometric
structural information from the manifold they formed.

To illustrate the generative capability of LGAE, we randomly generate
images using the model and plot them along with images generated from
VAE. Figures \ref{fig:latent-mnist} and \ref{fig:latent-svhn} show the
generated images from both models trained on the MNIST and SVHN
datasets, respectively.

\section{Conclusions}
\label{sec:conclusions}

We propose Lie group auto-encoder (LGAE), which is a
\emph{encoder-decoder} type of neural network model. Similar to VAE,
the proposed LGAE model has the advantages of generating examples from
the training data distribution, as well as mapping inputs to latent
representations. The Lie group structure of Gaussian distributions is
systematically exploited to help design the network. Specifically, we
design an exponential mapping layer, derive a Lie group intrinsic
loss, and propose to use Lie algebra vectors as latent
representations. Experimental results on the MNIST and SVHN datasets
testify to the effectiveness of the proposed method.


\bibliographystyle{icml2019}
\bibliography{ref}

\begin{thebibliography}{18}
\providecommand{\natexlab}[1]{#1}
\providecommand{\url}[1]{\texttt{#1}}
\expandafter\ifx\csname urlstyle\endcsname\relax
  \providecommand{\doi}[1]{doi: #1}\else
  \providecommand{\doi}{doi: \begingroup \urlstyle{rm}\Url}\fi

\bibitem[Deshpande et~al.(2017)Deshpande, Lu, Yeh, Chong, and
  Forsyth]{deshpande_learning_2017}
Deshpande, A., Lu, J., Yeh, M., Chong, M.~J., and Forsyth, D.
\newblock Learning {{Diverse Image Colorization}}.
\newblock In \emph{2017 {{IEEE Conference}} on {{Computer Vision}} and
  {{Pattern Recognition}} ({{CVPR}})}, pp.\  2877--2885, July 2017.

\bibitem[Doersch(2016)]{doersch_tutorial_2016}
Doersch, C.
\newblock Tutorial on {{Variational Autoencoders}}.
\newblock \emph{arXiv:1606.05908 [cs, stat]}, June 2016.

\bibitem[Donahue et~al.(2016)Donahue, Kr\"ahenb\"uhl, and
  Darrell]{donahue_adversarial_2016}
Donahue, J., Kr\"ahenb\"uhl, P., and Darrell, T.
\newblock Adversarial {{Feature Learning}}.
\newblock \emph{arXiv:1605.09782 [cs, stat]}, May 2016.

\bibitem[Duchi et~al.(2011)Duchi, Hazan, and Singer]{duchi_adaptive_2011}
Duchi, J., Hazan, E., and Singer, Y.
\newblock Adaptive {{Subgradient Methods}} for {{Online Learning}} and
  {{Stochastic Optimization}}.
\newblock \emph{Journal of Machine Learning Research}, 12\penalty0
  (Jul):\penalty0 2121--2159, 2011.

\bibitem[Gong et~al.(2009)Gong, Wang, and Liu]{gong_shape_2009}
Gong, L., Wang, T., and Liu, F.
\newblock Shape of {{Gaussians}} as feature descriptors.
\newblock In \emph{{{IEEE Conference}} on {{Computer Vision}} and {{Pattern
  Recognition}}}, pp.\  2366--2371, June 2009.

\bibitem[Goodfellow et~al.(2014)Goodfellow, {Pouget-Abadie}, Mirza, Xu,
  {Warde-Farley}, Ozair, Courville, and Bengio]{goodfellow_generative_2014}
Goodfellow, I., {Pouget-Abadie}, J., Mirza, M., Xu, B., {Warde-Farley}, D.,
  Ozair, S., Courville, A., and Bengio, Y.
\newblock Generative {{Adversarial Nets}}.
\newblock In \emph{Advances in {{Neural Information Processing Systems}}}, pp.\
  ~9, 2014.

\bibitem[Hinton \& Salakhutdinov(2006)Hinton and
  Salakhutdinov]{hinton_reducing_2006}
Hinton, G.~E. and Salakhutdinov, R.~R.
\newblock Reducing the {{Dimensionality}} of {{Data}} with {{Neural Networks}}.
\newblock \emph{Science}, 313\penalty0 (5786):\penalty0 504--507, July 2006.

\bibitem[Kingma \& Welling(2013)Kingma and Welling]{kingma_auto-encoding_2013}
Kingma, D.~P. and Welling, M.
\newblock Auto-{{Encoding Variational Bayes}}.
\newblock \emph{arXiv:1312.6114 [cs, stat]}, December 2013.

\bibitem[Knapp(2002)]{knapp_lie_2002}
Knapp, A.~W.
\newblock \emph{Lie {{Groups Beyond}} an {{Introduction}}}.
\newblock Progress in {{Mathematics}}. {Birkh\"auser Basel}, 2 edition, 2002.
\newblock ISBN 978-0-8176-4259-4.

\bibitem[Lecun et~al.(1998)Lecun, Bottou, Bengio, and
  Haffner]{lecun_gradient-based_1998-1}
Lecun, Y., Bottou, L., Bengio, Y., and Haffner, P.
\newblock Gradient-based learning applied to document recognition.
\newblock \emph{Proceedings of the IEEE}, 86\penalty0 (11):\penalty0
  2278--2324, November 1998.

\bibitem[Mao et~al.(2017)Mao, Li, Xie, Lau, Wang, and Smolley]{mao_least_2017}
Mao, X., Li, Q., Xie, H., Lau, R. Y.~K., Wang, Z., and Smolley, S.~P.
\newblock Least {{Squares Generative Adversarial Networks}}.
\newblock In \emph{International {{Conference}} on {{Computer Vision}}}, pp.\
  ~9, 2017.

\bibitem[Miyato et~al.(2018)Miyato, Kataoka, Koyama, and
  Yoshida]{miyato_spectral_2018}
Miyato, T., Kataoka, T., Koyama, M., and Yoshida, Y.
\newblock Spectral {{Normalization}} for {{Generative Adversarial Networks}}.
\newblock In \emph{International {{Conference}} on {{Learning
  Representations}}}, February 2018.

\bibitem[Netzer et~al.(2011)Netzer, Wang, Coates, Bissacco, Wu, and
  Ng]{netzer_reading_2011}
Netzer, Y., Wang, T., Coates, A., Bissacco, A., Wu, B., and Ng, A.~Y.
\newblock Reading {{Digits}} in {{Natural Images}} with {{Unsupervised Feature
  Learning}}.
\newblock In \emph{{{NIPS Workshop}} on {{Deep Learning}} and {{Unsupervised
  Feature Learning}}}, pp.\ ~9, 2011.

\bibitem[Paszke et~al.(2017)Paszke, Gross, Chintala, Chanan, Yang, DeVito, Lin,
  Desmaison, Antiga, and Lerer]{paszke_automatic_2017}
Paszke, A., Gross, S., Chintala, S., Chanan, G., Yang, E., DeVito, Z., Lin, Z.,
  Desmaison, A., Antiga, L., and Lerer, A.
\newblock Automatic differentiation in {{PyTorch}}.
\newblock In \emph{{{NIPS Workshop Autodiff}}}, October 2017.

\bibitem[Schlegl et~al.(2017)Schlegl, Seeb\"ock, Waldstein, {Schmidt-Erfurth},
  and Langs]{schlegl_unsupervised_2017}
Schlegl, T., Seeb\"ock, P., Waldstein, S.~M., {Schmidt-Erfurth}, U., and Langs,
  G.
\newblock Unsupervised {{Anomaly Detection}} with {{Generative Adversarial
  Networks}} to {{Guide Marker Discovery}}.
\newblock In Niethammer, M., Styner, M., Aylward, S., Zhu, H., Oguz, I., Yap,
  P.-T., and Shen, D. (eds.), \emph{Information {{Processing}} in {{Medical
  Imaging}}}, Lecture {{Notes}} in {{Computer Science}}, pp.\  146--157.
  {Springer International Publishing}, 2017.
\newblock ISBN 978-3-319-59050-9.

\bibitem[Tuzel et~al.(2008)Tuzel, Porikli, and Meer]{tuzel_pedestrian_2008}
Tuzel, O., Porikli, F., and Meer, P.
\newblock Pedestrian {{Detection}} via {{Classification}} on {{Riemannian
  Manifolds}}.
\newblock \emph{IEEE Transactions on Pattern Analysis and Machine
  Intelligence}, 30\penalty0 (10):\penalty0 1713--1727, October 2008.

\bibitem[Yeh et~al.(2016)Yeh, Liu, Goldman, and Agarwala]{yeh_semantic_2016}
Yeh, R., Liu, Z., Goldman, D.~B., and Agarwala, A.
\newblock Semantic {{Facial Expression Editing}} using {{Autoencoded Flow}}.
\newblock \emph{arXiv:1611.09961 [cs]}, November 2016.

\bibitem[Zhang et~al.(2018)Zhang, Goodfellow, Metaxas, and
  Odena]{zhang_self-attention_2018}
Zhang, H., Goodfellow, I., Metaxas, D., and Odena, A.
\newblock Self-{{Attention Generative Adversarial Networks}}.
\newblock \emph{arXiv:1805.08318 [cs, stat]}, May 2018.

\end{thebibliography}

\end{document}